\documentclass[journal]{IEEEtran}
\usepackage{paralist,tikz,graphicx,amsmath, amsfonts,amssymb,amsthm,overpic, latexsym,epsfig,color,rotating,paralist,times,float,subcaption,algorithm,verbatim,enumitem,url}
\usetikzlibrary{fit,positioning}
\usepackage{tikz,pgfplots,stfloats}

\usepackage[font=scriptsize]{caption}

  \usepackage{tikz}
\usetikzlibrary{patterns}

\usepackage{amsmath}
\usepackage{amssymb}
\usepackage{amsthm}
\usepackage{bbm}
\usepackage{color}
\usepackage{algorithm}
\usepackage{algpseudocode}
\usepackage{url}
\usepackage{mathtools}
\usepackage{xfrac}

\usepackage{enumitem}
\usepackage{geometry}
\usepackage{tikz}
\usetikzlibrary{calc,patterns,decorations.pathmorphing,decorations.markings,arrows}
\usetikzlibrary{quotes,arrows.meta}

\tikzstyle{block} = [draw, rectangle, minimum height=3em, minimum width=3em]
\tikzstyle{sum} = [draw, circle, node distance=1.5cm]
\tikzstyle{input} = [coordinate]
\tikzstyle{output} = [coordinate]

\newcommand{\state}{\mu}
\newcommand{\pre}[1]{\rho^{{(#1)}}}

\newcommand{\ie}{\emph{i.e.}}

\renewcommand{\leq}{\leqslant}
\renewcommand{\geq}{\geqslant}
\renewcommand{\epsilon}{\varepsilon}
\renewcommand{\phi}{\varphi}

\newtheorem{thm}{Theorem}

\newtheorem{remark}{Remark}

\title{\vspace{-1.3cm} Slow Convergence of  Interacting Kalman Filters in Word-of-Mouth Social Learning}
\author{Vikram Krishnamurthy, \and Cristian R. Rojas
\thanks{Vikram Krishnamurthy is with the  School of Electrical and Computer Engineering, Cornell University.
    Email: vikramk@cornell.edu. \\  Cristian Rojas is with the  Division of Decision and Control Systems, School of Electrical Engineering and Computer Science, KTH Royal Institute of Technology, Stockholm, Sweden.  Email: crro@kth.se  \\  Krishnamurthy's research was supported by the U.S.\ Army Research Office under grant
    W911NF-24-1-0083,  and National Science Foundation under grants CCF-2112457 and  CCF-2112457.
Rojas' research  was supported by the Swedish Research Council research environment NewLEADS, contract 2016-06079;
and Wallenberg AI, Autonomous Systems and Software Program (WASP), funded by Knut and Alice Wallenberg Foundation.}
}
\date{}

\begin{document}
\maketitle
\thispagestyle{empty}


\begin{abstract}
We consider word-of-mouth social learning involving $m$ Kalman filter agents that operate sequentially. The first Kalman filter receives the raw observations, while each subsequent  Kalman filter receives a noisy measurement of the conditional mean of the previous Kalman filter. The prior is updated by the $m$-th Kalman filter. When $m=2$, and the observations are noisy measurements of a Gaussian random variable,  the covariance goes to zero as $k^{-1/3}$ for $k$ observations, instead of  $O(k^{-1})$  in the standard Kalman filter. In this paper we prove that for $m$ agents, the covariance decreases to zero as $k^{-(2^m-1)}$, i.e, the learning slows down exponentially with the number of agents. We also show that by artificially weighing the prior at each time, the learning rate can be made optimal as $k^{-1}$. The implication is that in word-of-mouth social learning, artificially re-weighing the prior can yield the optimal  learning rate. 
\end{abstract}

\section{Introduction and Problem Formulation}

 {Social learning}  serves as a useful  mathematical abstraction for modeling  the interaction of social  sensors \cite{AO11,Ban92,BHW92,BMS20,Cha04,Kri16,KH15,KP14,Say14b,WD16}
It is  an integral part of human behavior and is studied widely in economics, marketing, political science and sociology (where the term groupthink is used), to model the behavior of financial markets, 
 social groups and social networks. Related models have been studied in  sequential decision making in
 computer science  and signal processing.
 In social learning,
 agents estimate the underlying state not only from their local noisy measurements, but also from the actions of previous
 agents.
 The agents are Bayesian; so estimating the state given the previous agent's action involves using Bayes rule to infer about Bayes rule. This double Bayesian inference results in unusual behavior.
 For example, when the state, observation and action spaces of individual agents are finite, then it can be shown under quite general conditions that the agents form an information cascade, namely, after a finite amount of time, agents blindly repeat the actions of previous agents regardless of their private observations.

 This paper analyzes social learning among interacting agents that observe the actions of previous agents in Gaussian noise. The agents aim to estimate an underlying Gaussian state.   The agents action is its state estimate. Each agent deploys a Kalman filter (albeit degenerate since the state is a random variable). Instead of information cascades, we show that social learning results in a significant slow down in the learning process.

 The model studied in this paper   is motivated by   learning by doing (in economics) involving  team decisions, and also word of mouth decision making \cite{Viv93,Viv97}. The book \cite{Cha04} provides a detailed exposition of social learning models. The asymmetric information structure in the model considered in this paper also arises in adversarial signal processing applications involving counter-autonomous systems
 \cite{KR19}.

\subsection{Problem Formulation}

Consider the following problem of social learning: Agent~$1$, at each time instant $k \in \mathbb{N}$, receives a signal $s_k = \theta + e_k$, where $\theta \sim \mathcal{N}(\bar{\theta}, \lambda_0)$ and $e_k \sim \mathcal{N}(0, \lambda_e)$ (independent of $\theta$). The posterior distribution of $\theta$ given $s_1, \dots, s_k$ is Gaussian with mean $\mu_k^{(1)}$ and precision $\rho_k^{(1)}$ (the precision is the inverse of the covariance). Therefore, based on Theorem~\ref{thm:cond_Gauss} from Appendix~1, these quantities are updated as follows (where $\rho_0^{(1)} := 1 / \lambda_0$ and $\rho_e := 1 / \lambda_e$):
\begin{align*}
\rho_k^{(1)} &= \rho_{k-1}^{(1)} + \rho_e \\
\mu_k^{(1)} &= (1 - \alpha_k^{(1)}) \mu_{k-1}^{(1)} + \alpha_k^{(1)} s_k, \qquad \alpha_k^{(1)} = \rho_e / \rho_k^{(1)}.
\end{align*}
Agent~$1$ transmits a noisy version of $\mu_k^{(1)}$, say, $\tilde{\mu}_k^{(1)} = \mu_k^{(1)} + w_k^{(1)}$, with $w_k^{(1)} \sim \mathcal{N}(0, \lambda_w^{(1)})$ (independent of $\theta$ and $e$) to Agent~$2$; note that since $\rho_k^{(1)}$ does not depend on the measurements, it can be assumed to be known to Agent~$2$.
In a more general setup, we can consider $m > 1$ agents, where Agent~$i$ ($i = 1, \dots, m - 1$) transmits a noisy version of its posterior $\mu_k^{(i)}$, say, $\tilde{\mu}_k^{(i)} = \mu_k^{(i)} + w_k^{(i)}$, with $w_k^{(i)} \sim \mathcal{N}(0, \lambda_w^{(i)})$ (independent of all previous random variables) to Agent~$i+1$.

\smallskip
In a word-of-mouth social learning problem, we assume that Agents $1, \dots, m-1$ rely on a ``publicly known'' posterior $(\mu_{k-1}^{(m)}, \rho_{k-1}^{(m)})$, which is generated by the last ($m$-th) agent. This is illustrated  in Figure~\ref{fig:kalman} for the case $m=3$ agents.

\begin{figure*}[h] \centering
 \scalebox{1}{
\begin{tikzpicture}[node distance=1.25cm and 2cm]
    \node[draw, thick, rectangle,align=center] (box1) {Kalman \\  Filter 1};
    \node[draw, thick, rectangle, right=of box1, align=center] (box2) {Kalman \\ Filter 2};
    \node[draw, thick, rectangle, right=of box2, align=center] (box2a) {Kalman \\ Filter 3};
    \node[draw, thick,  rectangle, above=of box1] (box3) {Prior};
    
    \draw[Latex-] (box1.west) -- ++(-0.5,0) node[left] {$s_k$};
    \draw[-Latex] (box1) -- (box2) node[midway, above]{{\small $\hat{\state}_k^{(1)}+w^{(1)}_k$}};
    \draw[-Latex] (box2) -- (box2a) node[midway, above]{{\small $\hat{\state}_k^{(2)}+w^{(2)}_k$}};
    \draw[-Latex] (box3) -- (box1) node[midway, right] {};
    \draw[-Latex] (box2a.east) -- ++(2,0) node[pos=0.5, above] {{\small $\hat{\state}_k^{(3)},\pre3_k$}} |-  (box3.east) ;
    \draw[-Latex] (box2a.east) -- ++(2,0) ;
    
    \draw[-Latex] (box3.south) -- ++(0,-0.5)  -|  (box2.north)  ;
     \draw[-Latex] (box3.south) -- ++(0,-0.5)  -|  (box2a.north)  ;
   \end{tikzpicture}}
 \caption{Three Interacting Kalman Filters} \label{fig:kalman}
\end{figure*}
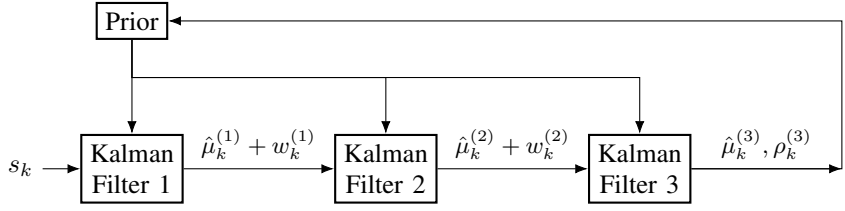

There is information asymmetry in the model: agent 1 receives the external observation $s_k$, while agent $m$ specifies the prior. Apart from the first agent, the remaining $m-1$  agents only receive word-of-mouth estimates from previous agents; specifically they do not obtain the observation $s_k$.

The aim of this paper is to characterize how the precision $\pre{m}_k$ evolves with time $k$. The main result
Theorem~\ref{thm:rho} shows that $\rho_k^{(m)} = O({k^{\sfrac{1}{(2^m - 1)}}})$ when there are $m$-interacting agents. Then Theorem~\ref{thm:rho_scaled} shows that we can recover the optimal convergence rate of $1/k$ by artificially scaling the prior covariance of all but the $m$-th agent.  


\section{Analysis}

In this section we analyze the word-of-mouth problem, as formulated in the previous section.
For help guide the reader's intuition,  we first start with $m=2$ agents (which was studied in \cite{Viv93,Viv97}). Then we discuss
$m=3$ agents. Finally we discuss the general case of $m$ agents. Throughout this paper, we consider scalar observations, and assume that unobserved state $\theta$ is a scalar Gaussian  random variable.

\subsection*{Two agents}
Consider first the case of two agents, \ie, $m=2$. Then, Agent~$2$ has access to a public posterior $(\mu_{k-1}^{(1)}, \rho_{k-1}^{(1)})$, then
\begin{align*}
\tilde{\mu}_k^{(1)}
&= (1 - \alpha_k^{(1)}) \mu_{k-1}^{(1)} + \alpha_k^{(1)} s_k + w_k^{(1)}\\
&= (1 - \alpha_k^{(1)}) \mu_{k-1}^{(1)} + \alpha_k^{(1)} \theta + \alpha_k^{(1)} e_k + w_k^{(1)},
\end{align*}
or
\begin{align*}
\frac{\tilde{\mu}_k^{(1)} - (1 - \alpha_k^{(1)}) \mu_{k-1}^{(1)}}{\alpha_k^{(1)}}
=: z_k^{(1)}
= \theta + e_k + \frac{w_k^{(1)}}{\alpha_k^{(1)}}.
\end{align*}
Observe that $z_k^{(1)}$ is \emph{informationally equivalent} to $\tilde{\mu}_k^{(1)}$. The equivalent ``noise'' is $n_k^{(1)} := e_k + \frac{1}{\alpha_k^{(1)}} w_k^{(1)}$, with variance $\lambda_e + \lambda_w^{(1)} / (\alpha_k^{(1)})^2$. Therefore, the posterior is updated by Agent~$2$ as follows:
\begin{align*}
\rho_k^{(2)}
&= \rho_{k-1}^{(2)} + \frac{1}{\lambda_e + \lambda_w^{(1)} / (\alpha_k^{(1)})^2} \\
&= \rho_{k-1}^{(2)} + \frac{1}{\lambda_e + \lambda_w^{(1)} \left(1 + \lambda_e \rho_{k-1}^{(1)}\right)^2}
\\
\mu_k^{(2)}
&= (1 - \alpha_k^{(2)}) \mu_{k-1}^{(2)} + \alpha_k^{(2)} z_k^{(1)}, \qquad \alpha_k^{(2)} \\ &= \frac{1}{\rho_k^{(2)} [\lambda_e + \lambda_w^{(1)} / (\alpha_k^{(1)})^2]}.
\end{align*}
If then the posterior of Agent~$2$ is publicly released, so that Agent~$1$ uses it instead of its own, that is, $\rho_{k-1}^{(1)}$ becomes $\rho_{k-1}^{(2)}$, the posterior precision of $\theta$, given data up to time $k$, satisfies the equation
\begin{align*}
\rho_k^{(2)}
= \rho_{k-1}^{(2)} + \frac{1}{\lambda_e + \lambda_w^{(1)} \left(1 + \lambda_e \rho_{k-1}^{(2)}\right)^2}.
\end{align*}
This setup, for two agents, is equivalent to the one in \cite[Section~3.3.1]{Cha04}, where it is shown that   $\rho_k$ does not increase linearly with $k$, but at the much slower rate $k^{1/3}$; the result was originally shown in \cite{Viv93,Viv97}.

\subsection*{Three agents}
Assume now that we have a third agent, who takes a noisy version of $\mu_k^{(2)}$, say, $\tilde{\mu}_k^{(2)} = \mu_k^{(2)} + w_k^{(2)}$, with $w_k^{(2)} \sim \mathcal{N}(0, \lambda_w^{(2)})$ (independent of all previously defined random variables); again, since $\rho_k^{(2)}$ does not depend on the measurements, it can be assumed to be known to Agent~$3$. If Agent~$3$ also has access to the public posterior $(\mu_{k-1}^{(1)}, \rho_{k-1}^{(1)})$, then
\begin{align*}
\tilde{\mu}_k^{(2)}
&= (1 - \alpha_k^{(2)}) \mu_{k-1}^{(1)} + \alpha_k^{(2)} z_k^{(1)} + w_k^{(2)}
\\ &= (1 - \alpha_k^{(2)}) \mu_{k-1}^{(1)} + \alpha_k^{(2)} \theta + \alpha_k^{(2)} n_k^{(1)} + w_k^{(2)},
\end{align*}
or
\begin{align*}
\frac{\tilde{\mu}_k^{(2)} - (1 - \alpha_k^{(2)}) \mu_{k-1}^{(1)}}{\alpha_k^{(2)}}
=: z_k^{(2)} = \theta + n_k^{(1)} + \frac{w_k^{(2)}}{\alpha_k^{(2)}}.
\end{align*}
The equivalent noise is now $n_k^{(2)} = n_k^{(1)} + \frac{w_k^{(2)}}{\alpha_k^{(2)}} = e_k + \frac{1}{\alpha_k^{(1)}} w_k^{(1)} + \frac{1}{\alpha_k^{(2)}} w_k^{(2)}$, which is a zero mean Gaussian variable with variance $\lambda_e + \lambda_w^{(1)} / (\alpha_k^{(1)})^2 + \lambda_w^{(2)} / (\alpha_k^{(2)})^2$. Then, the posterior is updated by Agent~$3$ as
\begin{align*}
\rho_k^{(3)}
&= \rho_{k-1}^{(3)} + \frac{1}{\lambda_e + \lambda_w^{(1)} / (\alpha_k^{(1)})^2 + \lambda_w^{(2)} / (\alpha_k^{(2)})^2} \\
&= \rho_{k-1}^{(3)} +
  \frac{1}{D_3}
\\
 D_3 &=  \lambda_e + \lambda_w^{(1)} \left(1 + \lambda_e \rho_{k-1}^{(1)}\right)^2  \\ & + \lambda_w^{(2)} \left( 1 + \rho_{k-1}^{(2)} \left[\lambda_e + \lambda_w^{(1)} \left(1 + \lambda_e \rho_{k-1}^{(1)}\right)^2\right] \right)^2
\\
\mu_k^{(3)}
&= (1 - \alpha_k^{(3)}) \mu_{k-1}^{(3)} + \alpha_k^{(3)} z_k^{(2)}, \\  \alpha_k^{(3)} &= \frac{1}{\rho_k^{(3)}[\lambda_e + \lambda_w^{(1)} / (\alpha_k^{(1)})^2 + \lambda_w^{(2)} / (\alpha_k^{(2)})^2]}.
\end{align*}
If now the posterior of Agent~$3$ is publicly released, so that Agents~$1$ and $2$ use it instead of their own, that is, $\rho_{k-1}^{(1)}$ and $\rho_{k-1}^{(2}$ are replaced by $\rho_{k-1}^{(3)}$, then the posterior precision of $\theta$, according to Agent~$3$, given data up to time $k$, satisfies the equation (where $D_3$ is defined above)
\begin{align*}
\rho_k^{(3)}
= \rho_{k-1}^{(3)} +
\frac{1}{D_3}
\end{align*}
To analyze the asymptotic behavior of $\rho_k^{(3)}$, at least from a heuristic perspective, we can assume that, for large $k$, $\rho_k^{(3)} \approx C k^\beta$; we will allow $C$ to take different values throughout the derivation. Then, the equation for the precision can be approximated as
\begin{align*}
C k^\beta
\approx C \sum_{i=1}^k
 \frac{1}{ i^{6 \beta}}
 \approx C k^{1 - 6 \beta}.
\end{align*}
Therefore, matching the exponents yields $\beta = 1 - 6 \beta$, or $\beta = 1/7$, so $\rho_k^{(3)}$ grows as $k^{1/7}$, even slower than when having only two agents.

\subsection*{m-agents}
If we now concatenate $m$ agents, where the $m$-th agent sets the public posterior, it is apparent that
\begin{align*}
\rho_{k-1}^{(j)} &= \rho_{k-1}^{(m)}, \quad j = 1, \dots, m-1, \\
n_k^{(j)} &= e_k + \sum_{i=1}^{j} \frac{1}{\alpha_k^{(i)}} w_k^{(i)} \sim \mathcal{N}\left(0, \lambda_e + \sum_{i=1}^{j} \frac{\lambda_w^{(i)}}{(\alpha_k^{(i)})^2}\right), \\ &  \qquad \qquad j = 1, \dots, m-1, \\
\rho_k^{(j)} &= \rho_{k-1}^{(j)} + \frac{1}{\lambda_e + \sum_{i=1}^{j-1} \lambda_w^{(i)} / (\alpha_k^{(i)})^2}, \quad j = 2, \dots, m, \\
\alpha_k^{(j)} &= \frac{1}{\rho_k^{(j)} \left[\lambda_e + \sum_{i=1}^{j-1} \lambda_w^{(i)} / (\alpha_k^{(i)})^2 \right]}, \quad j = 1, \dots, m.
\end{align*}
By defining $\gamma_k^{(j)} := \lambda_w^{(j)} / (\alpha_k^{(j)})^2$ as the new variance contribution to the equivalent noise entering agent $j+1$, we can write the equations for the precision as
\begin{align}
\rho_k^{(j)} &= \rho_{k-1}^{(m)} + \frac{1}{\lambda_e + \sum_{i=1}^{j-1} \gamma_k^{(i)}}, \quad j = 2, \dots, m, \label{eq:iterations1} \\
\gamma_k^{(j)} &= \lambda_w^{(j)} (\rho_k^{(j)})^2 \left[\lambda_e + \sum_{i=1}^{j-1} \gamma_k^{(i)} \right]^2 \nonumber\\ & \hspace{-1cm} = \lambda_w^{(j)} \left[1 + \rho_{k-1}^{(m)} \left(\lambda_e + \sum_{i=1}^{j-1} \gamma_k^{(i)}\right) \right]^2, \, j = 1, \dots, m-1. \label{eq:iterations2}
\end{align}
To get some intuition regarding the asymptotic behavior of $\rho_k^{(m)}$, assume that $\rho_k^{(m)} \approx C k^\beta$. Then,
\begin{align*}
\gamma_k^{(1)} &\approx C (\rho_k^{(m)})^2 = C k^{2 \beta}, \\
\gamma_k^{(2)} &\approx C (\rho_k^{(m)})^2 (\gamma_k^{(1)})^2  = C k^{(2 + 2 \cdot 2) \beta} = C k^{6\beta}, \\
\gamma_k^{(3)} &\approx C (\rho_k^{(m)})^2 (\gamma_k^{(1)})^2  = C k^{(2 + 2 \cdot 6) \beta} = C k^{14 \beta},
\end{align*}
and by induction on $j$ it follows that $\gamma_k^{(j)} \approx C k^{(2 \cdot 2^j - 2) \beta}$. Therefore, the equation for $\rho_k^{(m)}$ implies that $C k^{\beta} \approx \rho_k^{(m)} \approx C \sum_{i=1}^k i^{(2-2^m) \beta} = C k^{1 + (2-2^m) \beta}$. Matching exponents, this gives $\beta = 1 + (2-2^m) \beta$, or $\beta = 1 / (2^m - 1)$. In other words, for $m$ agents,
\begin{align*}
\rho_k^{(m)} \approx C k^{\sfrac{1}{(2^m - 1)}}.
\end{align*}
The following theorem formalizes this result:

\begin{thm} \label{thm:rho}
In the word-of-mouth social learning problem, the public belief has a precision $\rho_k^{(m)}$ satisfying
\begin{multline} \label{eq:rho_Riccati}
  \rho_k^{(m)} = \rho_{k-1}^{(m)} +  \frac{1}{D_m} \\
D_m = 
  \lambda_e^{2^{m-1}} \left[ \prod_{i=1}^{m-1} [\lambda_w^{(i)}]^{2^{m-1-i}} \right] \left(\rho_{k-1}^{(m)}\right)^{2^m - 2} \\  + f\left(\left(\rho_{k-1}^{(m)}\right)^{2^m - 3}\right),
\end{multline}
where $f\colon \mathbb{R}_+ \to \mathbb{R}_+$ is strictly increasing and such that there exists an $M > 0$ for which $\lambda_e \leq f(x) \leq M x$ for all $x > 0$. Furthermore,
\begin{align} \label{eq:asymptotic_rho}
\lim_{k \to \infty} \frac{\rho_k^{(m)}}{k^{\sfrac{1}{(2^m - 1)}}}
= \left[ \frac{2^m - 1}{\lambda_e^{2^{m-1}} \prod_{i=1}^{m-1} [\lambda_w^{(i)}]^{2^{m-1-i}}} \right]^{\sfrac{1}{(2^m - 1)}}.
\end{align}
\end{thm}

\begin{remark}
Slow learning occurs regardless of the noise variance  that corrupts the estimates of the Kalman filters. Of course, if the noise variance is zero, then the precision  grows as $O(k)$, consistent with the classical Kalman filter. Thus, the model where agents  exactly know the estimate of the previous agent, is not robust;  even small errors cause the learning rate (precision)  to slow from $O(k)$ to $O(k^{\sfrac{1}{(2^m - 1)}})$.
\end{remark}

\begin{proof}
We will first prove, by induction on $j = 1, \dots, m$, that $\gamma_k^{(j)}$ is a polynomial of order $2^{j+1} - 2$ in $\rho_{k-1}^{(m)}$, with leading coefficient $a_j = \lambda_e^{2^j} \prod_{i=1}^j [\lambda_w^{(i)}]^{2^{j-i}}$ and all remaining coefficients being positive.
For $j = 1$, $\gamma_k^{(1)}$ is indeed a polynomial of order $2$ in $\rho_{k-1}^{(m)}$, with leading coefficient $a_1 = \lambda_w^{(1)} \lambda_e^2$ and all remaining coefficients being positive. Assume now that the statement holds for $j = 1, \dots, p$ ($1 \leq p < k - 1$); then, \eqref{eq:iterations2} implies that $\gamma_k^{(p+1)}$ is also a polynomial in $\rho_{k-1}^{(m)}$ of order $2 (1 + 2^{p+1} - 2) = 2^{p+2} - 2$, leading coefficient
\begin{align*}
a_{p+1}
&= \lambda_w^{(p+1)} a_p^2
= \lambda_w^{(p+1)} \lambda_e^{2 \cdot 2^p} \prod_{i=1}^p [\lambda_w^{(i)}]^{2 \cdot (2^{p-i})}
\\ &= \lambda_e^{2^{p + 1}} \lambda_w^{(p+1)} \prod_{i=1}^p [\lambda_w^{(i)}]^{2^{p+1-i}}
= \lambda_e^{2^{p + 1}} \prod_{i=1}^{p+1} [\lambda_w^{(i)}]^{2^{p+1-i}},
\end{align*}
and all remaining coefficients being positive.
Plugging this result into \eqref{eq:iterations1} yields \eqref{eq:rho_Riccati}, where $f\left(\left(\rho_{k-1}^{(m)}\right)^{2^m - 3}\right)$ is a polynomial in $\rho_{k-1}^{(m)}$ of degree strictly less than $2^m - 2$ with positive coefficients, hence it is monotonically increasing in $\rho_{k-1}^{(m)}$, or, equivalently, in $\left(\rho_{k-1}^{(m)}\right)^{2^m - 3}$. Furthermore, this implies, together with \eqref{eq:iterations1}, that there exists an $M > 0$ such that $\lambda_e \leq f(x) \leq M x$ for all $x > 0$, thus proving \eqref{eq:rho_Riccati}.

Eq.~\eqref{eq:asymptotic_rho} follows from Theorem~\ref{thm:riccati} of Appendix~2 (with $\delta = 0$).
\end{proof}

As Theorem~\ref{thm:rho} states, the variance of the public belief of the agents, corresponding to $1 / \rho_k^{(m)}$, decays to zero as $k \to \infty$, as expected. However, the decay rate is $k^{-\sfrac{1}{(2^m - 1)}}$, which is much  slower compared to the optimal rate, $k^{-1}$, for any $m \geq 2$. The learning rate slows exponentially with the number of agents $m$.


\section{Recovering the optimal convergence rate}

One approach to recover the optimal rate $1/k$ of the posterior public variance is to scale the public precision $\rho_{k-1}^{(m)}$ by a power of time, $k^{-\delta}$ (for some $\delta \in [0, 1]$), before being used by all agents \emph{except for the last one}. In this case, Eqn.~\eqref{eq:iterations2} is replaced by
\begin{align} \label{eq:iterations_scaled}
\gamma_k^{(j)} &= \lambda_w^{(j)} (k^{-\delta} \rho_k^{(j)})^2 \left[\lambda_e + \sum_{i=1}^{j-1} \gamma_k^{(i)} \right]^2 \\ &= \lambda_w^{(j)} \left[1 + k^{-\delta} \rho_{k-1}^{(m)} \left(\lambda_e + \sum_{i=1}^{j-1} \gamma_k^{(i)}\right) \right]^2,\nonumber \\ & \qquad \qquad  j = 1, \dots, m - 1.  \nonumber
\end{align}
The next theorem is a modification of Theorem~\ref{thm:rho} which accounts for the scaling of the precision.

\begin{thm} \label{thm:rho_scaled}
In the word-of-mouth social learning problem with scaling of the precision, the public belief has a precision $\rho_k^{(m)}$ satisfying
\begin{align} \label{eq:rho_Riccati_scaled}
  \rho_k^{(m)} &= \rho_{k-1}^{(m)} + \frac{1}{T_0} \\
  T_0 &= \lambda_e^{2^{m-1}} \left[ \prod_{i=1}^{m-1} [\lambda_w^{(i)}]^{2^{m-1-i}} \right] \left(k^{-\delta} \rho_{k-1}^{(m)}\right)^{2^m - 2} \nonumber\\ & + f\left(\left(k^{-\delta} \rho_{k-1}^{(m)}\right)^{2^m - 3}\right), \nonumber
\end{align}
where $f\colon \mathbb{R}_+ \to \mathbb{R}_+$ is strictly increasing and such that there exists an $M > 0$ for which $\lambda_e \leq f(x) \leq M x$ for all $x > 0$. Furthermore, if $\delta < 1$, then
\begin{align} \label{eq:asymptotic_rho_scaled}
&\lim_{k \to \infty} \frac{\rho_k^{(m)}}{k^{\sfrac{[1 + \delta (2^m - 2)]}{(2^m - 1)}}} \\
&= \left[ \frac{2^m - 1}{\lambda_e^{2^{m-1}} [1 + \delta(2^m - 2)] \prod_{i=1}^{m-1} [\lambda_w^{(i)}]^{2^{j-i}-1}} \right]^{\sfrac{1}{(2^m - 1)}},
\end{align}
while if $\delta = 1$,
\begin{align} \label{eq:asymptotic_rho_scaled_delta1}
\lim_{k \to \infty} \frac{\rho_k^{(m)}}{k}
= C,
\end{align}
where $C$ is the unique solution to $f(C) = 1/C - \lambda_e^{2^{m-1}} C^N \prod_{i=1}^{m-1} [\lambda_w^{(i)}]^{2^{j-1-i}}$.
\end{thm}

\begin{proof}
The derivation of \eqref{eq:rho_Riccati_scaled} is similar to that of \eqref{eq:rho_Riccati} in Theorem~\ref{thm:rho}. The asymptotic behavior of $\rho_k^{(m)}$ then follows directly from Theorem~\ref{thm:riccati} of Appendix~2.
\end{proof}

\begin{remark}
Theorem~\ref{thm:rho_scaled} shows that, for $0 \leq \delta \leq 1$, the rate of $\rho_k^{(m)}$ is $k^{\sfrac{[1 + \delta (2^m - 2)]}{(2^m - 1)}}$, which is maximum for $\delta = 1$, yielding the rate $k^1$.
\end{remark}

\begin{remark} \label{rem:delta_large}
As the rate is monotonically increasing in $\delta$, one might be tempted to consider $\delta > 1$ in order to obtain rates faster that $k^1$. However, from \eqref{eq:iterations1} it follows that $\rho_k^{(m)} \leq \rho_{k-1}^{(m)} + 1 / \lambda_e$, which implies that $\rho_k^{(m)}$ cannot grow faster than $k / \lambda_e$. On the other hand, for $\delta > 1$, due to the previous inequality, the denominator of the second term in \eqref{eq:rho_Riccati_scaled} converges to $\lambda_e$ as $k \to \infty$ at a rate $k^{1-\delta}$; this means that $\rho_k^{(m)} / k \to 1 / (\lambda_e + \lambda_w^{(1)} + \cdots + \lambda_w^{(m-1)})$ as $k \to \infty$.
\end{remark}

\begin{remark}
An extreme version of the scaling suggested in Remark~\eqref{rem:delta_large} consists in zeroing out the public precision $\rho_{k-1}^{(m)}$ for all agents except for the last one. In this case, \eqref{eq:iterations2} becomes $\gamma_k^{(j)} = \lambda_w^{(j)}$ for $j = 1, \dots, m-1$, so \eqref{eq:iterations2} yields $\rho_k^{(m)} = \rho_{k-1}^{(m)} + 1 / (\lambda_e + \lambda_w^{(1)} + \cdots + \lambda_w^{(m-1)})$, which can be iterated to deliver $\rho_k^{(m)} = \rho_0^{(m)} + k / (\lambda_e + \lambda_w^{(1)} + \cdots + \lambda_w^{(m-1)})$. This corresponds to the same rate as the one obtained with scaling using any $\delta > 1$ (but with an easier-to-compute proportionality constant), and it can be intuitively explained as ``turning off'' the filtering of Agents $1$ to $m-1$, so they can be equivalently seen as a single agent measuring a noisy version of $\theta$, with variance $\lambda_e + \lambda_w^{(1)} + \cdots + \lambda_w^{(m-1)}$.
\end{remark}

\begin{remark}
As an alternative scaling approach, one could consider letting the un-scaled agent being not the last one, but a different agent, or even un-scaling more than one agent. However, if the last agent uses a scaled version of $\rho_{k-1}^{(m)}$, irrespective of whether the remaining agents receive a scaled or un-scaled version of $\rho_{k-1}^{(m)}$, then modifying \eqref{eq:rho_Riccati_scaled} leads to the inequality $\rho_k^{(m)} \leq k^{-\delta} \rho_{k-1}^{(m)} + 1/\lambda_e$. For every $k > 1$ the right hand side of this inequality can be upper bounded by $\mu \rho_{k-1}^{(m)} + 1/\lambda_e$, where $0 < \mu < 1$, which leads to the linear inequality $\rho_k^{(m)} \leq \mu \rho_{k-1}^{(m)} + 1/\lambda_e$ that, in turn, can be iterated to show that $(\rho_k^{(m)})$ is a positive bounded sequence; therefore, in this case the posterior variance does not even converge to $0$ as $k \to \infty$.

In addition, if any of the agents $1, \dots,m-1$ does not receive a scaled version of $\rho_{k-1}^{(m)}$ (whereas the last agent uses the un-scaled precision $\rho_{k-1}^{(m)}$), then iterating \eqref{eq:iterations_scaled} leads to a modified version of \eqref{eq:rho_Riccati_scaled} where the denominator of the second term in its right hand side is a polynomial in $\rho_{k-1}^{(m)}$ whose leading term has a term $k^{-\delta}$ of power strictly smaller than $2^m - 2$ (because these factors come directly from the scaling of Agents $1, \dots, m-1$). Therefore, one needs a value of $\delta$ strictly larger than $1$ to attain the optimal rate $k^1$ for the precision $r_k^{(m)}$.
\end{remark}

\appendix


\section*{Appendix 1: Conditional Gaussian distributions}

This is a simple result on conditional Gaussian distributions that appears in \cite{Cha04}:

\begin{thm} \label{thm:cond_Gauss}
Let $\theta$ be a real random variable with prior distribution $\mathcal{N}(\bar{\theta}, \sigma_0^2)$. Consider a measurement $s = \theta + e$, where $e \sim \mathcal{N}(0, \sigma_e^2)$.
Then, the posterior distribution of $\theta$ given $s$ is $\mathcal{N}(\bar{\theta} + (s - \bar{\theta}) \sigma_0^2 / (\sigma_0^2 + \sigma_e^2), \sigma_0^2 \sigma_e^2 / (\sigma_0^2 + \sigma_e^2))$, or $\mathcal{N}((1 - \alpha) \bar{\theta} + \alpha s, \sigma^2)$, where $\alpha = \sigma_0^2 / (\sigma_0^2 + \sigma_e^2)$ and $\sigma^2 = \sigma_0^2 \sigma_e^2 / (\sigma_0^2 + \sigma_e^2)$. If one uses ``precision'' $\rho = 1 / \sigma^2$ instead of variance $\sigma^2$, this can be written as $\rho = \rho_0 + \rho_e$ and $\alpha = \sigma^2 / \sigma_e^2 = \rho_e / \rho$.
\end{thm}

\begin{proof}
Note that
\begin{align*}
\begin{bmatrix}
\theta - \bar{\theta} \\
s - \bar{\theta}
\end{bmatrix} &\sim
\mathcal{N} \left(
0,
\begin{bmatrix}
1 & 0 \\
1 & 1
\end{bmatrix}
\begin{bmatrix}
\sigma_0^2 & 0 \\
0 & \sigma_e^2
\end{bmatrix}
\begin{bmatrix}
1 & 1 \\
0 & 1
\end{bmatrix}
\right)
\\ &= \mathcal{N} \left(
0,
\begin{bmatrix}
\sigma_0^2 & \sigma_0^2 \\
\sigma_0^2 & \sigma_0^2 + \sigma_e^2
\end{bmatrix}
\right).
\end{align*}
We want to decompose the covariance matrix of $[\theta, s]^T$ as
\begin{align*}
\begin{bmatrix}
\sigma_0^2 & \sigma_0^2 \\
\sigma_0^2 & \sigma_0^2 + \sigma_e^2
\end{bmatrix} =
\begin{bmatrix}
a & b \\
0 & c
\end{bmatrix}
\begin{bmatrix}
a & 0 \\
b & c
\end{bmatrix}.
\end{align*}
By direct computation, we obtain $c = \sqrt{\sigma_0^2 + \sigma_e^2}$, $b = \sigma_0^2 / c = \sigma_0^2 / \sqrt{\sigma_0^2 + \sigma_e^2}$ and $a = \sqrt{\sigma_0^2 - b^2} = \sqrt{\sigma_0^2 - \sigma_0^4 / (\sigma_0^2 + \sigma_e^2)} = \sigma_0 \sigma_e / \sqrt{\sigma_0^2 + \sigma_e^2}$. Therefore,
\begin{align*}
\begin{bmatrix}
\sigma_0^2 & \sigma_0^2 \\
\sigma_0^2 & \sigma_0^2 + \sigma_e^2
\end{bmatrix}
&= \begin{bmatrix}
\sigma_0 \sigma_e / \sqrt{\sigma_0^2 + \sigma_e^2} & \sigma_0^2 / \sqrt{\sigma_0^2 + \sigma_e^2} \\
0 & \sqrt{\sigma_0^2 + \sigma_e^2}
\end{bmatrix} \\&  \times
\begin{bmatrix}
\sigma_0 \sigma_e / \sqrt{\sigma_0^2 + \sigma_e^2} & 0 \\
\sigma_0^2 / \sqrt{\sigma_0^2 + \sigma_e^2} & \sqrt{\sigma_0^2 + \sigma_e^2}
\end{bmatrix} \\
& \hspace{-2cm} = \begin{bmatrix}
1 & \sigma_0^2 / (\sigma_0^2 + \sigma_e^2) \\
0 & 1
\end{bmatrix}
\begin{bmatrix}
\sigma_0^2 \sigma_e^2 / (\sigma_0^2 + \sigma_e^2) & 0 \\
0 & \sigma_0^2 + \sigma_e^2
\end{bmatrix} \\ \times & 
\begin{bmatrix}
1 & 0 \\
\sigma_0^2 / (\sigma_0^2 + \sigma_e^2) & 1
\end{bmatrix}
\end{align*}
This shows that $\theta = \bar{\theta} + (s - \bar{\theta}) \sigma_0^2 / (\sigma_0^2 + \sigma_e^2) + w$, where $w \sim \mathcal{N}(0, \sigma_0^2 \sigma_e^2 / (\sigma_0^2 + \sigma_e^2))$ is independent of $s$, thus proving the result.
\end{proof}



\section*{Appendix 2: Asymptotic behavior of precision}

\begin{thm} \label{thm:riccati}
Let $(X_k)_{k \in \mathbb{N}}$ be a sequence of positive real numbers satisfying the recursive equation
\begin{align} \label{eq:Riccati_scaled}
X_k &= X_{k-1} + \frac{1}{C (k^{-\delta} X_{k-1})^N + f((k^{-\delta} X_{k-1})^{N-1})}, \\ & \qquad \qquad k \in \mathbb{N}, \nonumber
\end{align}
where $N \in \mathbb{N}$, $0 \leq \delta \leq 1$ and $C > 0$, and there are values $m, M > 0$ such that $f\colon \mathbb{R}_+ \to \mathbb{R}_+$ satisfies $m < f(x) \leq M x$ for all $x > 0$. If $\delta = 1$, assume in addition that $f$ is strictly increasing. Then, $X_k / k^{(1 + \delta N)/(N+1)} \xrightarrow{k \to \infty} [(N+1)/((1 + \delta N) C)]^{1/(N+1)}$ if $\delta < 1$, and $X_k / k \xrightarrow{k \to \infty} \bar{Y}$ if $\delta = 1$, where $\bar{Y}$ is the unique solution to $f(\bar{Y}) = 1/\bar{Y} - C \bar{Y}^N$.
\end{thm}

\begin{proof}
Firstly, note that the fraction in \eqref{eq:Riccati_scaled} is positive, so $(X_k)$ is monotonically increasing. Also, it is not bounded: if there existed an $A > 0$ such that $X_k \leq A$ for all $k$, then the fraction in \eqref{eq:Riccati_scaled} would be lower bounded by $1/(C A^N + M A^{N-1}) > 0$, implying that $X_k \to \infty$, which contradicts the boundedness of $(X_k)$.

\smallskip
We will proceed as in \cite[Sec3.5]{Cha04}.  Let $X_k = Y_k k^{(1 + \delta N)/(1 + N)}$, where $(Y_k)$ is another positive sequence which we shall prove to converge to a positive, finite number. We will first show that $(Y_k)$ is a bounded sequence, so it contains a converging sub-sequence; then, we will establish that every converging sub-sequence of $(Y_k)$ has the same limit, thus proving that $(Y_k)$ is convergent.

Equation~\eqref{eq:Riccati_scaled} can be re-written as 
\begin{multline*}
Y_k k^{(1 + \delta N)/(1 + N)} = Y_{k-1} (k-1)^{(1 + \delta N)/(1 + N)} + 
\frac{1}{T_1} \\
T_1 = 
C Y_{k-1}^N k^{-\delta N} (k-1)^{N(1 + \delta N)/(1 + N)} \\ + f(Y_{k-1}^{N-1} k^{-\delta (N-1)} (k-1)^{(N-1) (1 + \delta N)/(1 + N)}),
\end{multline*}
or
\begin{align} \label{eq:Riccati2_scaled}
Y_k &= Y_{k-1} \left( 1 - \frac{1}{k} \right)^{(1 + \delta N)/(1 + N)} + 
 \frac{1}{T_2} \\
T_2 &=
  C Y_{k-1}^N k \left(1 - \frac{1}{k} \right)^{(N + \delta N^2)/(1 + N)} \nonumber\\ &+ k^{(1 + \delta N)/(1 + N)} f\big(Y_{k-1}^{N-1} k^{(1 - \delta) (N-1)/(N+1)}\nonumber \\ &\times  \left(1 - \frac{1}{k} \right)^{(N-1) (1 + \delta N)/(1 + N)}\big)  \nonumber.
\end{align}
%
This equation implies that
\begin{align*}
Y_k \leq Y_{k-1} \left( 1 - \frac{1 + \delta N}{(N+1)k} + \mathcal{O}(k^{-2}) \right) + \frac{1}{k^{(1 + \delta N)/(N + 1)} m},
\end{align*}
which shows that there is a $K_1 \in \mathbb{N}$ large enough so that if $k \geq K_1$ then
\begin{align*}
Y_k \leq Y_{k-1} \left( 1 - \frac{1 + \delta N}{2 (N + 1) k} \right) + 1,
\end{align*}
so if $k \geq K_1$ and $Y_{k-1} \geq 1$, then $Y_k < 2$.

\smallskip
Now, notice that \eqref{eq:Riccati2_scaled} also implies that
\begin{align*}
Y_k &\leq Y_{k-1} \left( 1 - \frac{1}{k} \right)^{(1 + \delta N)/(1 + N)} \\ &+ C^{-1} Y_{k-1}^{-N} k^{-1} \left(1 - \frac{1}{k} \right)^{-(N + \delta N^2)/(1 + N)},
\end{align*}
or
\begin{align*}
Y_k &\leq Y_{k-1} \left( 1 - \frac{1 + \delta N}{(1 + N)} k^{-1} + \mathcal{O}(k^{-2}) \right) \\ &+ C^{-1} Y_{k-1}^{-N} \left(k^{-1} + \frac{N + \delta N^2}{(1 + N)} k^{-2} + \mathcal{O}(k^{-3}) \right).
\end{align*}
There exists a $K_2 \in \mathbb{N}$ sufficiently large such that $2 / (K_2 C) < 1$ and, for all $k \geq K_2$, if $Y_{k-1} > 1$,
\begin{align} \label{eq:ineq1_scaled}
Y_k \leq Y_{k-1} \left( 1 - \frac{1 + \delta N}{2(1 + N) k} \right) + \frac{2}{C k Y_{k-1}^N}.
\end{align}
From this inequality and the choice of $K_2$, we see that there exists a $Y > 1$ large enough such that, if $k \geq K_2$ and $Y_{k-1} \geq Y$, then
\begin{align*}
Y_k \leq Y_{k-1} \left( 1 - \frac{1 + \delta N}{3(1 + N) k}\right).
\end{align*}
Therefore, for all $k > K := \max(K_1,K_2)$, it holds that $Y_k \leq 2 Y$, because if $Y \leq Y_{k-1} \leq 2Y$, the previous inequality shows that $Y_k < Y_{k-1} < 2Y$, while, if $1 < Y_{k-1} < Y$, \eqref{eq:ineq1_scaled} implies that $Y_k \leq Y_{k-1} + 2/(k C) < Y + 1 < 2Y$, and if $Y_{k-1} \leq 1$, the previous discussion established that $Y_k < 2 < 2Y$. In summary, the sequence $(Y_k)$ is upper bounded by $\max \{ Y_1, \dots, Y_K, 2 Y\}$.

\smallskip
Consider a converging sub-sequence $(Y_{k_n})$, with limit $\bar{Y}$, and let $(Y_{k'_n})$ be a sub-sequence of $(Y_{k_n})$ such that $(Y_{k'_n-1})$ is convergent with limit $\bar{Y}'$. Then, letting $n \to \infty$ in \eqref{eq:Riccati2_scaled} along the sub-sequence $k'_n$ gives $\bar{Y} = \bar{Y}'$. This means that every converging sub-sequence of $(Y_k)$ has the same limit. Furthermore, if $\delta < 1$, equating the $k^{-1}$ terms of both sides of \eqref{eq:Riccati2_scaled},
\begin{align*}
0 = -\frac{(1 + \delta N) \bar{Y}}{N+1} + \frac{1}{C \bar{Y}^N},
\end{align*}
or $\bar{Y} = [(N+1)/((1 + \delta N) C)]^{1/(N+1)}$. If $\delta = 1$, equating the $k^{-1}$ terms of both sides of \eqref{eq:Riccati2_scaled} yields, instead,
\begin{align*}
0 = -\bar{Y} + \frac{1}{C \bar{Y}^N + f(\bar{Y})},
\end{align*}
or $f(\bar{Y}) = 1/\bar{Y} - C \bar{Y}^N$, which has a unique solution, because $\bar{Y} \mapsto 1/\bar{Y} - C \bar{Y}^N$ is strictly decreasing in $(0, +\infty)$ and its range in $(0, +\infty)$ is $\mathbb{R}$, while $f$ is strictly increasing. This concludes the proof.
\end{proof}


\bibliographystyle{abbrv}
\bibliography{$HOME/texstuff/styles/bib/vkm}

\end{document}